\documentclass[letterpaper, 10 pt, conference]{ieeeconf}
\IEEEoverridecommandlockouts
\usepackage{times}
\usepackage[utf8]{inputenc}
\usepackage[T1]{fontenc}
\usepackage{import}
\usepackage{bbm}
\usepackage{graphicx}
\usepackage{dblfloatfix}
\usepackage{makecell}
\usepackage[export]{adjustbox}
\usepackage[cmex10]{amsmath}
\usepackage{amssymb}

\usepackage{amsthm}
\usepackage{cite}
\usepackage[caption=false,font=footnotesize]{subfig}
\usepackage[usenames, dvipsnames]{color}
\usepackage{overpic}
\usepackage[section]{placeins}
\usepackage{complexity}
\usepackage{balance}
\usepackage{siunitx}
\usepackage[hidelinks]{hyperref}
\graphicspath{{./figures/}}

\usepackage{bbm}
\usepackage{algorithm}
\usepackage{float}
\newfloat{algorithm}{t}{lop}

\PassOptionsToPackage{noend}{algpseudocode}
\usepackage{algpseudocode}
\usepackage{cleveref}
\newtheorem{theorem}{Theorem}
\newtheorem{corollary}[]{Corollary}
\newtheorem{observation}[]{Observation}

\newcommand{\revision}[1]{{\color{black}#1}}

\newcommand{\delete}[1]{{}}
\newcommand{\BILLE}{\textsc{Bill-E} }

\crefname{figure}{Fig.}{Figs.}
\crefname{theorem}{Theorem}{Theorems}
\crefname{corollary}{Corollary}{Corollaries}
\crefname{observation}{Observation}{Observations}
\crefname{algorithm}{Algorithm}{Algorithms}

\usepackage[normalem]{ulem}

\usepackage[switch]{lineno}  

\makeatletter
\newcommand*{\algrule}[1][\algorithmicindent]{\makebox[#1][l]{\hspace*{.5em}\vrule height .75\baselineskip depth .25\baselineskip}}%

\newcount\ALG@printindent@tempcnta
\def\ALG@printindent{%
    \ifnum \theALG@nested>0
        \ifx\ALG@text\ALG@x@notext
            \addvspace{-3pt}
        \else
            \unskip
            \ALG@printindent@tempcnta=1
            \loop
                \algrule[\csname ALG@ind@\the\ALG@printindent@tempcnta\endcsname]%
                \advance \ALG@printindent@tempcnta 1
            \ifnum \ALG@printindent@tempcnta<\numexpr\theALG@nested+1\relax
            \repeat
        \fi
    \fi
    }%
\usepackage{etoolbox}
\patchcmd{\ALG@doentity}{\noindent\hskip\ALG@tlm}{\ALG@printindent}{}{\errmessage{failed to patch}}
\makeatother

\begin{document}
\raggedbottom
\author{Javier Garcia$^{1}$, Michael Yannuzzi$^{1}$, Peter Kramer$^{2}$, Christian Rieck$^{2}$, S\'andor P.~Fekete$^{2}$, and Aaron T. Becker$^{1,2}$
}%
\title{
\LARGE\bf Reconfiguration of a 2D Structure Using Spatio-Temporal Planning and Load Transferring
\vspace{-1em}
\thanks{
This work was supported by the National Science Foundation under  \href{http://nsf.gov/awardsearch/showAward?AWD_ID=1553063}{[IIS-1553063},
\href{https://nsf.gov/awardsearch/showAward?AWD_ID=1849303}{1849303},
\href{https://nsf.gov/awardsearch/showAward?AWD_ID=2130793}{2130793]},
the DFG project „Space Ants“, grant number FE 407/22-1,
and the Alexander von Humboldt Foundation.
}
\thanks{
$^{1}$Electrical Engineering, University of Houston, TX USA\newline {\tt\small \{jgarciagonzalez,mcyannuzzi,atbecker\}@uh.edu}}
\thanks{
$^{2}$Computer Science, TU Braunschweig, Braunschweig, Germany {\tt\small \{kramer,rieck,fekete\}@ibr.cs.tu-bs.de}}
}%
\maketitle
\begin{abstract}
We present progress on the problem of reconfiguring a 2D arrangement of building material by a cooperative \revision{group} of robots.
These robots must avoid collisions, deadlocks, and are subjected to the constraint of maintaining connectivity of the structure.
We develop two reconfiguration methods, one based on spatio-temporal planning, and one based on target swapping, to increase building efficiency. 
The first method can significantly reduce planning times compared to other multi-robot planners. The second method helps to reduce the amount of time robots spend waiting for paths to be cleared, and the overall distance traveled by the robots.
\end{abstract}

\section{Introduction}\label{sec:Intro}

A challenge in robotics is to use agents (robots) to change the configuration of a supply of
passive building material. A typical task arises from relocating 
a collection of tiles from a given
start configuration into a desired goal configuration in an efficient manner. 
Reconfiguration time can be decreased by using multiple robots,
but this requires careful coordination to avoid collisions, deadlocks, imbalanced task allocation between robots, as well
as maintaining important constraints such as
connectivity of the structure.

These considerations play an important role when 
constructing \emph{large-scale} configurations,
ranging from plans for kilometer-scale manufacturing structures in
space~\cite{jenett2017design,abdel2020space}, to millimeter-scale smart
material~\cite{thalamy2021engineering}, and nano-scale assembly with
DNA~\cite{song2017reconfiguration}. 
In such domains, where disconnected components can drift apart, it is often necessary that the structure remains a single component.

In previous work~\cite{single-bille-reconfig-IROS}, we showed how a sampling-based
approach (the RRT$^*$) can be used to enable a single robot to reconfigure a 2D set of connected tiles in complex environments, where multiple obstacles may be present. 
\revision{In this paper, we assume that the reconfiguration sequence is already given, by RRT* or another approach, and focus on multi-robot cooperation to carry it out}. This requires methods for coordination and task allocation to improve reconfiguration times without disconnecting the structure.

\Cref{fig:NewLeading} shows the hardware platform that motivates the constraints in this theoretical work.
Two \BILLE bots, a platform developed by Jenett et al.~\cite{jenett2019material,ultralight24}, are standing on a tile structure. 
\revision{These inchworm robots step on the tiles, using a rotating key on the bottom of their feet to lock on the structure. A gripper located on the front foot can move up and down to pick and place tiles, and the robots can move while carrying the tiles}.
In order to guarantee that the robots can reach any point in the structure, the latter must remain connected throughout the reconfiguration.

\begin{figure}[t]
\setlength{\abovecaptionskip}{0pt}
\centering
    \adjincludegraphics[width=1\columnwidth,trim={{0.0\width} {0.0\width} {0.0\width} {0.0\width}},clip] {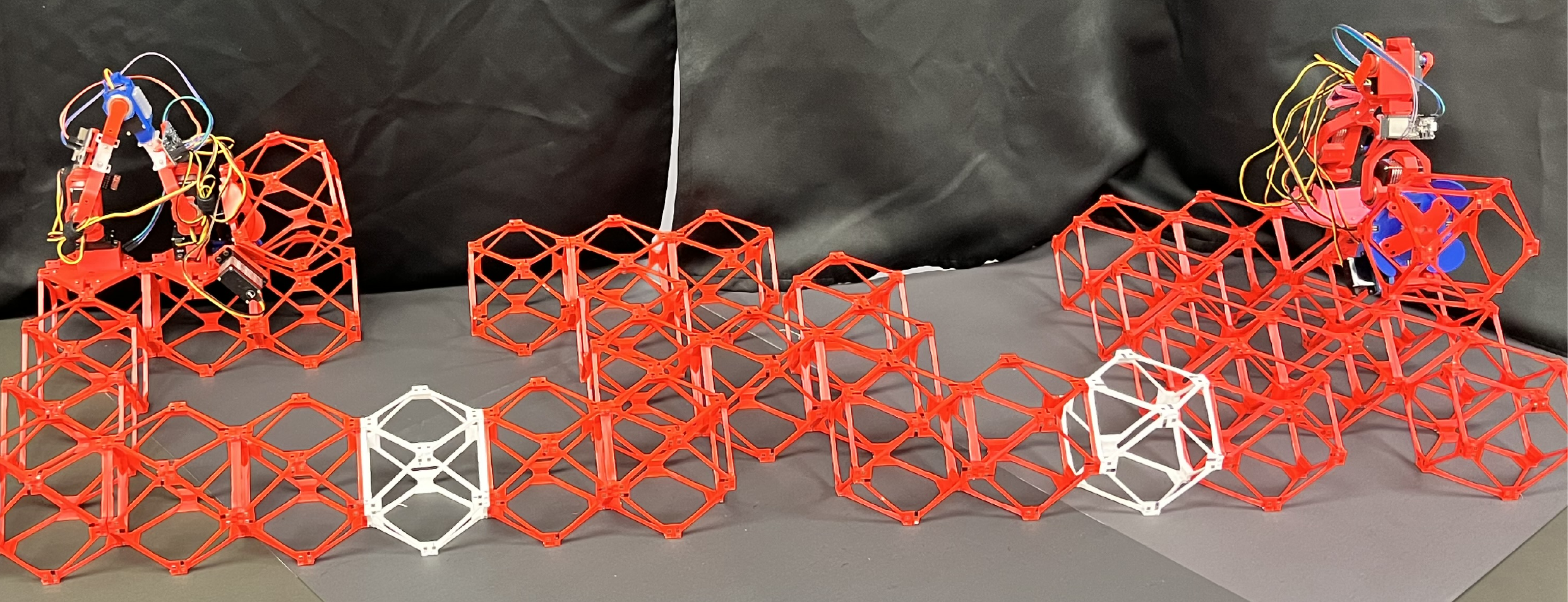}
    \adjincludegraphics[width=1\columnwidth,trim={{0\width} {0.5cm\width} {0\width} {0.8cm\width}},clip] {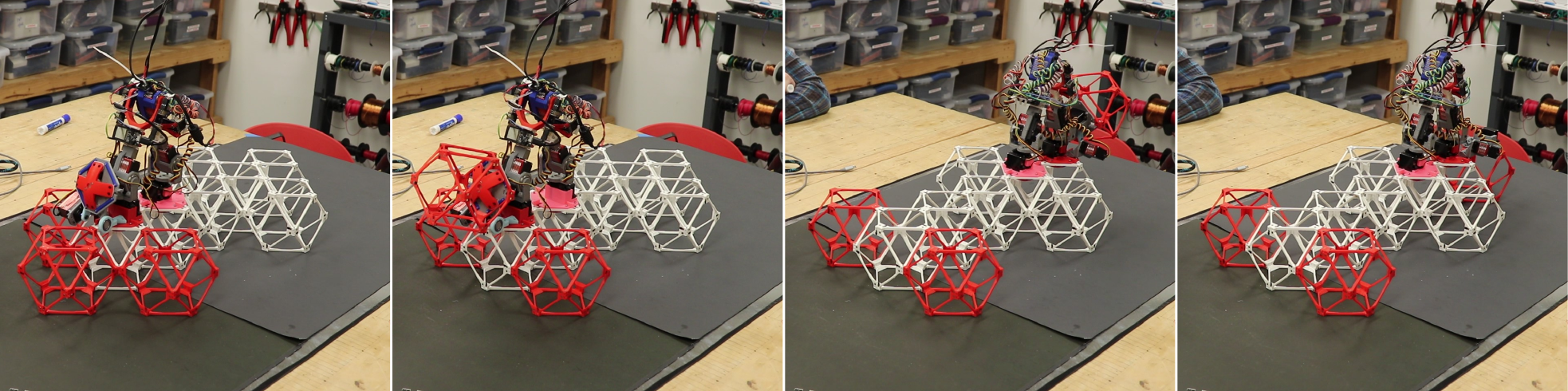}
    \caption{\label{fig:NewLeading}
    This paper implements and compares algorithms for automated reconfiguration using \revision{multiple robots}. Above is a physical representation~of the problem, and a sequence of frames showing a \BILLE bot moving a~tile. See video overview at~\href{https://youtu.be/tCKMjhkzbp8}{https://youtu.be/tCKMjhkzbp8}.
    }
   \vspace{-1em}
\end{figure}

\section{Related Work}\label{sec:RW}

\subsection{Automated reconfiguration}
\label{subsec:automated-reconfiguration}
In~\cite{jenett2019material}, the \BILLE bot traverses and alters a tile structure. 
For a single robot, the assembly sequence is deterministic, and is handled layer by layer in the three-dimensional case. 
The multi-robot case is handled by treating other robots as obstacles.  
If a desired path is blocked, a robot waits it to be cleared/constructed before placing their tile. 
In non-trivial configurations, this approach results in waiting times that increase with the number of robots.

Work on simplifying complex reconfiguration using principles of finite automata is presented in~\cite{abdel2020space},~\cite{fekete2022connected} and~\cite{NiesReconfig}. 
Structures can be built, scaled and rotated while respecting the constraint of tile connectivity. 

In our case, we exploit the ability to temporarily store tiles at any valid location when reconfiguring structures. 
One benefit is that  this enables creating bridges that act as shortcuts.
Deterministically finding these intermediate structures is challenging, so in~\cite{single-bille-reconfig-IROS} we used a sampling-based approach. 
The reconfiguration sequence is updated as lower cost paths are found. 
While we have only considered one robot (moving only one tile at a time), this paper focuses on performing reconfiguration sequences with~$m\in \mathbb{N}$ robots.

\subsection{Multi-robot planning}
\label{subsec:multi-robot-planning}
While increasing the number of robots from 1 to $m$ has potential benefits, it also causes additional difficulties. 
A~formulation as a centralized path planning problem raises the DoF by a factor of $m$.
The increased complexity makes sampling-based approaches harder to compute.
Moreover, the practical use of \revision{multiple robots} relies on \emph{distributed} methods for path planning and motion control.

To reduce the complexity of the problem, de-coupled planning methods have been developed.
Otte and Correll~\cite{Otte2018subspace} segment the joint problem and thus facilitate path computations; segments are only combined if no solution can be found.
Wagner and Choset~\cite{Choset2011Mstar} combine segmentation with~A$^*$ to create the M$^*$ algorithm, only considering neighbors when expanding from a given vertex.
Additionally, collision information is propagated to find paths without collisions.
Furthermore, a modification to the D$^*$ Lite algorithm is given by Pent et al.~\cite{Peng2015Dstar}, where robots are treated as obstacles and graph costs are updated as the agents move.

\subsection{\revision{Multi-agent cooperation}}

Increased reconfiguration speed can be achieved by considering swarms of cooperating robots, such as described in~\cite{FeketeKKRS23-journal-connected,isaac/FeketeKRS022,dfk+-cmprs-19}.
In addition, we can employ the robots' ability to swap targets as well as tasks. 
A \BILLE can, e.g., transfer its carried tile to a second \BILLE and proceed to a new target if doing so reduces the overall reconfiguration cost. 

We take inspiration from~\cite{Wang2020swap}, where Wang and Rubenstein use local task swapping to create a shape with a homogeneous swarm by using a hop-count algorithm.
Compared to their application, task swapping for the \BILLE bots is more complicated, since they cannot transfer carried tiles directly in their current version.
Instead, the tile must be placed where another robot can later pick it up.

\section{Computational Complexity}\label{subsec:LowerBounds}

We consider the \emph{workspace} to be a rectangular unit grid, where each cell is either free, filled by a tile, or filled by an obstacle.
We are interested in reconfiguring a set of tiles from a start to a goal configuration.
Both are \emph{connected} components, i.e., for every tile there is a path for the robot on the present configuration to every other tile. 
These shapes are called \emph{polyominoes}. 
As neither a robot nor a tile carried by a robot can cross an obstacle, we assume that start and goal configurations are located within the same connected component of free space.
\revision{With this assumption, there always exists a feasible reconfiguration sequence.}

An arrangement is reconfigured by moving a robot to an adjacent position of a tile, picking it up, walking along a path on the remaining configuration, then placing the tile in another location.
An ordered series of these operations is called a \emph{reconfiguration sequence}.
Distances between workspace positions are defined by the length of the geodesic edge-connected path between them, determined by a \emph{breadth-first search} tree over the configuration.
We call the problem of deciding whether the exists a sequence of a specific length that reconfigures a start into a goal configuration by a single robot the \textsc{\BILLE reconfiguration} problem.
If more than a single robot is considered, i.e., we have $m$ robots available, we refer to this as the \textsc{multi-robot} or \textsc{cooperative} variant.
We make use of the \textsc{Hamiltonian path} problem in grid graphs to prove \NP-hardness of deciding the length of optimal reconfiguration sequences.
This problem, which asks us to decide whether there exists a path in the input grid graph that visits each vertex exactly once, was shown to be \NP-complete by Itai et al.~\cite{ItaiPS82}.

Before arguing hardness of the reconfiguration problem, we observe that we can certify a correct solution in polynomial time, implying membership in \NP.
Clearly, it is possible to confirm the validity of an individual move in polynomial time with the number of tiles and robots involved.
Validating an arbitrary reconfiguration sequence is therefore possible by individually validating the steps involved.
\begin{observation}
    We can validate solutions in polynomial time, so \textsc{\BILLE reconfiguration} is in~\NP.
\end{observation}

\begin{theorem}\label{thm:bille-reconf-np-complete}
    \textsc{\BILLE reconfiguration} is \NP-complete.
\end{theorem}
\begin{proof}
    For an illustration of our construction, we refer to~\cref{fig:np-hardness}. Given a grid graph $G=(V,E)$, we construct a start and a goal configuration as follows.
    
    We begin by creating an arbitrarily scaled embedding of the vertex positions in $V$ in the grid.
    To achieve this, we place \emph{vertex tiles} (black) in identical locations in both the start and goal configuration, such that two adjacent vertices $u,v\in V$ will be exactly $s+7$ units apart in the configurations,
    as we combine a scale factor $s\in\mathbb{N}_+$ with a requirement of $7\times7$ units of space around each vertex tile for our construction.
    
    \begin{figure}[tb]
        \centering
        \def\svgwidth{\columnwidth}
        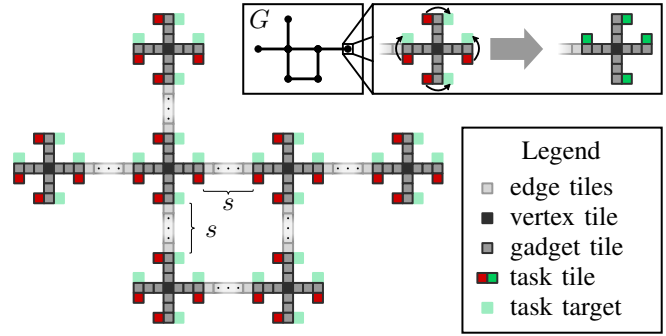
        \caption{Symbolic overview of the construction used to show \NP-hardness of the \textsc{\BILLE reconfiguration} problem, see proof of~\cref{thm:bille-reconf-np-complete}.}
        \label{fig:np-hardness}
        \vspace{-1em}
    \end{figure}
    
    Each of the vertex tiles is then surrounded by a so-called \emph{vertex gadget}, which contains four \emph{task tiles} and \emph{task positions} (red and green) arranged around a $7\times 7$ cross of auxiliary \emph{gadget tiles} (dark gray), centered on the vertex tile itself.
    A task tile or position in this construction represents a location that is occupied exclusively in either the start or the goal configuration, respectively.
    We say that a reconfiguration sequence \emph{solves} a vertex gadget exactly if it picks up each task tile and places it at one of the nearby task positions.
    Therefore, a reconfiguration sequence that transforms the start into the goal configuration must solve each gadget.
    Finally, we connect every pair of gadgets that correspond to adjacent vertices in $G$ by a straight path of $s$ many so-called \emph{edge tiles} (light gray), which are identical in the start and goal configurations.
    
    Let $t_g$ now refer to the number of moves required to solve one vertex gadget.
    By construction, this value is equal for all gadgets.
    A reconfiguration sequence that solves all vertex gadgets therefore takes at least $t_g\cdot |V|$ moves, not accounting for travel between vertices.
    As this takes $(s+7)$ moves per vertex by construction, we conclude that there exists a sequence consisting of exactly $(s+7)\cdot (|V|-1)+ t_g\cdot |V|$ moves, if there is a Hamiltonian path in $G$.
    In this case, no reconfiguration sequence could ever be faster, as building bridges between any two vertex tiles would take more than~$(s+7)$ moves, therefore resulting in a more expensive reconfiguration sequence than on a Hamiltonian path.
    
    If no Hamiltonian path exists, a shorter reconfiguration sequence that travels solely along the existing edges of the underlying graph would visit at least one vertex tile twice, implying a longer sequence.
    If the robot were to construct at least one additional bridge between vertices, this would immediately imply that it spent at least $(s+7)$ moves constructing a bridge in addition to the $(s+7)\cdot (|V|-1)+ t_g\cdot |V|$ moves it would have to spend solving the instance.
\end{proof}
By creating $m$ copies of the resulting configurations (one for each robot), placing these pairwise sufficiently far apart, and connecting them by a single line of tiles (to create a connected configuration),~\cref{cor:multi-robot-hardness} is straightforward.
\begin{corollary}\label{cor:multi-robot-hardness}
    \textsc{Multi-robot \BILLE reconfiguration} with $m$ robots is \NP-complete for every $m\in \mathbb{N}$.
\end{corollary}
Clearly, every \textsc{Multi-robot \BILLE reconfiguration} instance is an instance of the \textsc{Cooperative \BILLE reconfiguration} problem as well, meaning that~\Cref{cor:multi-robot-hardness} explicitly covers cooperative reconfiguration.
However, in our reduction, cooperation between robots is not necessary but rather sub-optimal. 
By adding a simple gadget to the above-mentioned construction, we are able to force cooperation and thus provide the following explicit result.

\begin{theorem}\label{thm:cooperative-hardness}
    The \textsc{Cooperative \BILLE reconfiguration} is \NP-complete even for two robots.
\end{theorem}

\begin{proof}
For the following we refer to~\cref{fig:cooperative-primitive}.
\begin{figure}[tb]
    \centering
    \vspace{1em}
    \def\svgwidth{.8\columnwidth}
    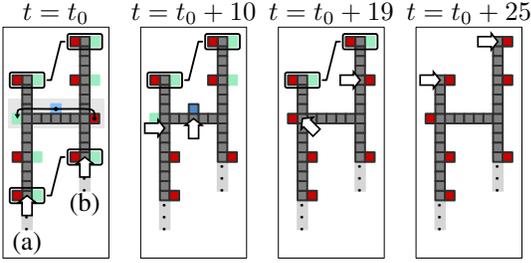
    \caption{The gadget that is used in the proof of~\cref{thm:cooperative-hardness}, showing that the cooperative variant of the problem is NP-hard.}
    \label{fig:cooperative-primitive}
    \vspace{-1.5em}
\end{figure}
We want the robots to enter this part of the construction simultaneously at time step $t_0$ at the positions labeled by arrows in the left figure.
Below (a) we place an instance of the construction given in the proof of~\cref{thm:bille-reconf-np-complete}, with a total of $|V|$ vertices and scale $s$.
Attached to (b) is an arbitrarily arranged string of $|V|$ task groups connected by edges of length $s$, such that an optimal schedule solves the gadgets in order of their position on the path.
Hence, the robots can arrive at the labeled positions ($t=t_0$) simultaneously, if and only if a Hamiltonian path exists in the input instance of~(a), implying an optimal schedule as shown in the proof of~\cref{thm:bille-reconf-np-complete}.

The cooperative gadget consists of one large and six small tasks, each requiring a tile to be moved by eight units or two units, respectively.
We demonstrate a schedule which can solve this gadget in $25$ moves~(see~\cref{fig:cooperative-primitive} left to right) by having each robot solve the three small tasks on the side they come in on, while splitting the large task into equal halves (indicated by the blue square).
If a shorter schedule were to exist, it would require a different assignment of tasks to robots.
If one robot were to solve four small tasks at two moves each, this would require it to travel four units in between tasks, in addition to crossing the bridge at least once.
Reaching and crossing the bridge costs an additional $10$ moves, resulting in a total makespan of at least $30$.
Note that picking up tiles and taking them away from their local tasks cannot accelerate this, as the number of moves required per task remains the same (or increases).
\end{proof}

\section{Methods}\label{sec:Methods}

\definecolor{myBlueRed}{RGB}{129,131,251}
\definecolor{myPink}{RGB}{252,130,130}
\definecolor{myGreen}{RGB}{132,252,133}
\definecolor{myGray}{RGB}{127,127,127}
\definecolor{myYellow}{RGB}{253,226,134}
\definecolor{myLavendar}{RGB}{225,133,252}
\definecolor{myCyan}{RGB}{133,226,253}

\begin{figure}[t]
\centering
    \adjincludegraphics[width=0.7\columnwidth] {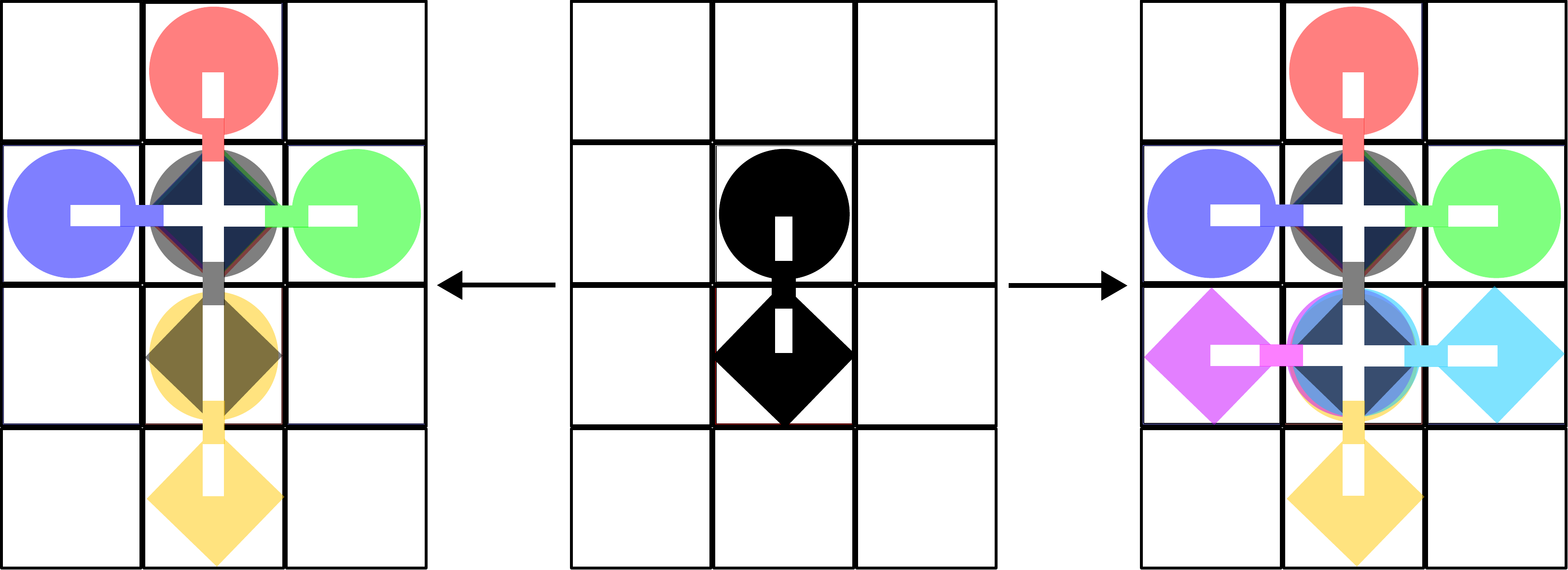}
    \caption{\label{fig:Neighbors}
    Starting from a configuration (middle), \BILLE can reach several configurations after a single motion. For clarity, the front foot is drawn as a circle, and configurations are color coded. (Left) The motions in $S_5$ are 
    \textcolor{myGray}{$\bullet$} waiting, 
    \textcolor{myPink}{$\bullet$} stepping forward, 
    \textcolor{myYellow}{$\bullet$} backward, 
    and placing the front foot one tile to the right
    \textcolor{myGreen}{$\bullet$} or one tile to the left \textcolor{myBlueRed}{$\bullet$} and moving the back foot one tile forward.
     (Right) $S_7$ is expanded to include placing the back foot one tile to the left 
    \textcolor{myLavendar}{$\bullet$} or right
    \textcolor{myCyan}{$\bullet$} and moving the front foot one tile backward. 
    }
\end{figure}

\subsection{Graph representation}

Many problems in path planning can be solved by converting them to graph search problems.
The configurations of the system can be discretized into vertices connected by edges or transitions, based on the basic motions of the robots.
Then, it becomes a matter of finding a valid path between the start and goal vertices along the provided edges.

A popular deterministic search algorithm is A$^*$, employed often due to its completeness and ease of implementation~\cite{lavalle2006planning}. 
Every iteration, A$^*$ expands a vertex by adding its \emph{neighbors} (vertices that can be reached with one transition) to an open list.
The most promising vertex in this list is then expanded and removed from the list.
This is repeated until the goal is reached.
The promise of vertex $v$ is defined as
\begin{equation}\label{eq:og_cost}
 f(v)=g(v)+\varepsilon \cdot h(v).
\end{equation}
Here, $g(v)$ is the cost from the start to $v$, and $h(v)$ is a heuristic \revision{(both admissible and consistent)} to estimate the cost from $v$ to the goal.
Scaling the heuristic by $\varepsilon >1$ speeds up the pathfinding process, at the cost of optimality~\cite{pohl1970first}. 

A significant drawback of such a  graph representation, however, is that the number of neighbors each vertex has rapidly increases for systems with more degrees of freedom, making it impractical for systems with large numbers of robots.
To deal with this, the M$^*$ variant initially only considers the \emph{limited neighbors} of vertices when expanding them~\cite{Choset2011Mstar}.
This approach is equivalent to planning the paths of the robots independently.
If a collision results, all vertices leading to it are updated with a collision set, and re-added to the open list.
When expanding vertices with collision sets, the paths of the colliding robots are planned jointly.
Notably, M$^*$ is most useful in situations when only a subset of robots are in danger of colliding on their optimal paths to their goals. 
On crowded maps, where the majority of vertices have collision sets, it reverts to regular A$^*$.

In order to employ graph search methods, we must now precisely define the \BILLE configurations and the feasible transitions for the graph representation.
Each vertex $v$ represents a configuration by a tuple of the coordinates of the robot's feet and of the carried tile, if applicable.
Neighbors of $v_i$ are all the configurations that the \BILLE can reach after one edge transition starting from the respective vertex.
All edge transitions depend on $S$, the robot's basic motions.

One advantage of using a graph search approach is that changes in hardware can be easily accounted for by changing~$S$. 
For example, our current \BILLE has a rotation limit of about 180$^\circ$ on its feet. 
This means that if we want one foot to be able to rotate the robot around, it can only do so in one direction. 
While our \revision{robots} are homogeneous, a heterogeneous \revision{group with different movement capabilites} can be handled by suitably modifying~$S$. 
Changes in~$S$ considerably impact the ability to find a solution.

Two sets of $S$, one with five motions and the other with seven, for a robot not carrying a tile are shown in~\cref{fig:Neighbors}. 
Although a larger $S$ means that the robots are capable of moving into more configurations, it also increases the time spent creating and searching the graph. 
Hence it can be beneficial to limit the number of basic motions.

\subsection{Temporal A$^*$}

We implement a temporal version of A$^*$ to handle the multi-robot problem with a time-varying structure. We use a \emph{time horizon} $H$ and a \emph{priority queue} $Q={1,2,...,m}$ as the de-coupling methods to reduce the complexity of the problem. The planner determines the next $H$ moves of each robot in order according to $Q$. Similar to~\cite{Peng2015Dstar} and \cite{silver2005cooperative}, when path planning for one robot, the other robots are treated as obstacles. 
Higher priority robots, those with IDs appearing first in $Q$, only consider the positions of lower priority robots at $H=1$. 
If $H>1$, the lower priority robots are assumed to move out of the way so their positions are not considered. 
Lower priority robots consider the positions of higher priority robots at the $H$ value they are planning for. If the current $Q$ results in a deadlock, the ID of the lower priority robot is shifted up in $Q$ and the paths are planned again. A higher $H$ value can correct deadlock situations but is more computationally expensive. \revision{It is also possible for a deadlock to be unavoidable for all permutations of $Q$ and any value of $H$, making our Temporal A$^*$ not complete. This is more likely to happen with smaller $S$.}

Time is embedded into the graph by differentiating vertices that describe the same configuration for a given robot but at different values of time $t$. This way, \emph{waiting} can be treated as any other member of $S$. All members of $S$ can be made to take the same amount of time by the central computer. Even if some motions are faster than others, the robots are signaled to wait for all others to be done moving. This synchronous scheme allows us to assign uniform costs to~$S$. Since \emph{waiting} is as costly as moving, the planner is incentivized to try to move robots through different paths to reach their destinations. Having many vertices that lead to the same configurations at different times makes the neighbor sets grow quickly, further limiting the practical values of $H$. 

Collisions are handled at edge transitions. If the transition involves a rotation, the tiles that the robot sweeps over are included in the \emph{collision set} $c$. \revision{If the robot is carrying a tile, $c$ is appropriately expanded.} If $c$ for two robots intersect, then there will be a collision, and the transition is labelled as invalid. In this case the lower priority robot will re-plan as explained before. Visuals of different $c$ are provided in~\cref{fig:CollSets}.

\begin{figure}[tb]
\centering
    \includegraphics[width=0.6\columnwidth] {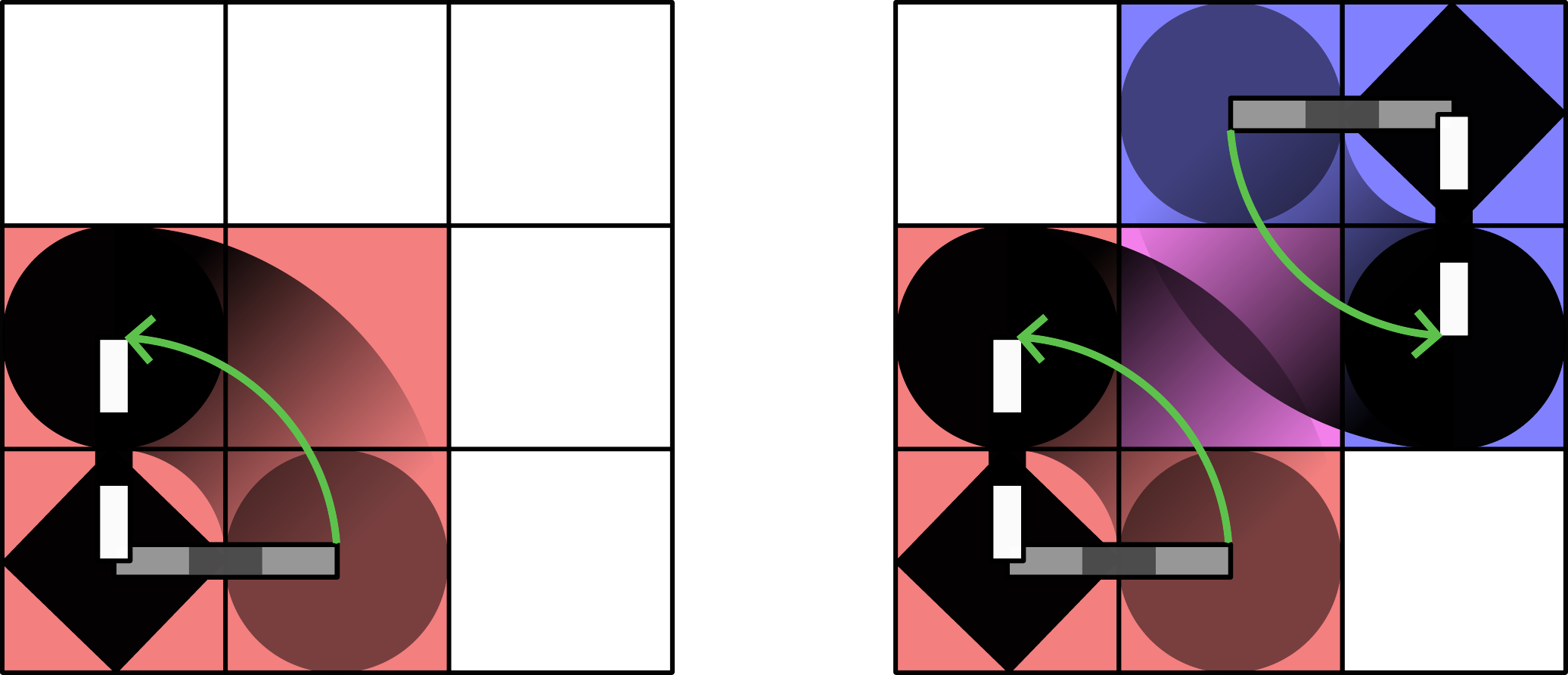}
    \caption{\label{fig:CollSets}
    The collision set of a motion is used to determine validity. (Left)~The \BILLE rotates $90^\circ$ counterclockwise, and the highlighted squares represent the respective collision set. (Right) Example of a collision resulting from moving two \BILLE bots in a certain way.
    }
\end{figure}

\subsection{Load transfer}

\begin{figure}[tb]
\centering
    \def\svgwidth{\columnwidth}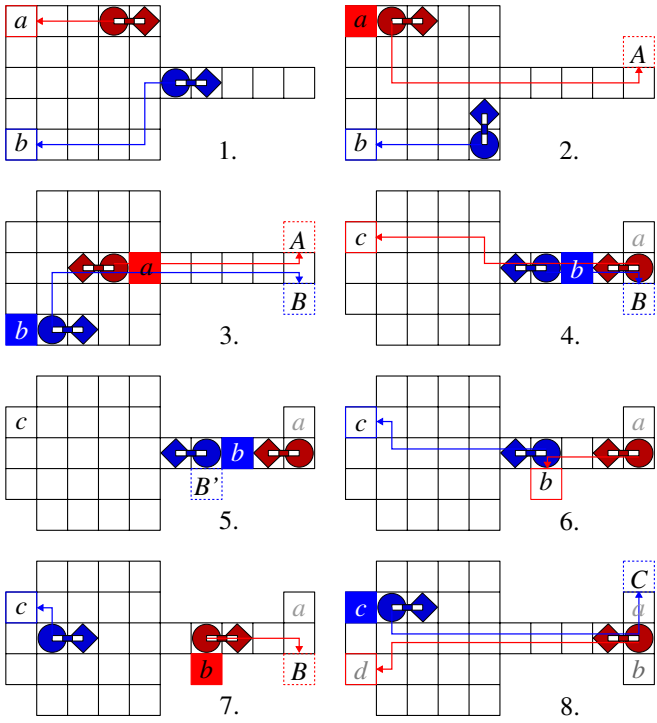
    \caption{\label{fig:LoadTransfer}
    1. Each robot moves to pick up its assigned tile. 2. Red picks up tile \emph{a} and is routed to place it at \emph{A}. 3. Blue picks up tile \emph{b} and is routed to place it at \emph{B}. 4. Red places tile \emph{a} and is routed to pick up tile \emph{c}. 5. The condition for load transfer is met at the new \emph{B\!'}. 6.~Blue places tile \emph{b} and is routed to tile \emph{c} that Red was originally going to pick up, while Red is routed to tile \emph{b}. 7. Red picks up tile \emph{b} and is routed to drop it off. 8. The building continues as normal.
    }
\end{figure}

A key aspect of our automated reconfiguration scheme is the ability of the robots to cooperate by exchanging carried tiles.
This allows robots to avoid collisions and deadlocks by changing targets, usually to pick-up and drop-off locations closer to their current location, reducing the density of paths.

Reducing the number of tiles traveled is the metric used to determine if a load transfer should take place.
The cost notation is as follows: $P_{i,(a,b)}^{(x,y)}$ is the cost for robot $i$ to pick-up a tile at location $(x,y)$ starting from $(a,b)$.
Similarly, $D_{i,(a,b)}^{(x,y)}$ is the cost for the same robot to drop-off a tile.
Two robots $i$ and $j$, one carrying a tile to $(x_1,y_1)$ starting from $(a_1,b_1)$, and the other on its way to pick up a tile at $(x_2,y_2)$ starting from $(a_2,b_2)$, have an original combined cost of
\begin{equation}\label{eq:DP_cost}
 C_{i,j}=D_{i,(a_1,b_1)}^{(x_1,y_1)}+P_{j,(a_2,b_2)}^{(x_2,y_2)}.
\end{equation}

The load transfer process for the \BILLE bots involves finding secondary locations to drop-off and pick-up their loads since the robots cannot exchange tiles directly in their current version. Following the case above, a third location $(x_3,y_3)$, with an adjacent tile $(a_3,b_3)$ is considered for the transfer. If the condition
\begin{equation}\label{eq:cost_comp1}
D_{i,(a_1,b_1)}^{(x_3,y_3)}+P_{i,(a_3,b_3)}^{(x_2,y_2)}+P_{j,(a_2,b_2)}^{(x_3,y_3)}+D_{j,(a_3,b_3)}^{(x_1,y_1)}<C_{i,j}
\end{equation}
is met, then robot $i$ moves to place its tile at $(x_3,y_3)$, where robot $j$ will pick it up later. Both robots then proceed to the other's original targets. 
In the case that both robots are initially carrying a tile, there is no need to switch tiles because the tiles are identical. 
Switching drop-off locations may shorten the required travel, i.e., if
\begin{equation}\label{eq:cost_comp2}
 \begin{aligned}
  D_{i,(a_1,b_1)}^{(x_2,y_2)}+D_{j,(a_2,b_2)}^{(x_1,y_1)}<D_{i,(a_1,b_1)}^{(x_1,y_1)}+D_{j,(a_2,b_2)}^{(x_2,y_2)}.
 \end{aligned}
\end{equation}

Load transfer is particularly useful in situations where a narrow path is present, such as a one-tile width bridge where two robots cannot pass each other.
Without transferring, one robot must wait for the other to get out of the way.
An example of~\eqref{eq:cost_comp1} being met is shown in~\cref{fig:LoadTransfer}. 

Finding candidate locations where load transfer can take place involves checking valid positions where the tile can be placed, usually along the outer edges of the polyomino. 
A~candidate location closer to the robot dropping off the tile is better, since dropping off a tile takes time. \revision{Structure connectivity remains a priority when performing load transfers.} 

Because this strategy can be computationally expensive with larger configurations, we limit the search to pairs of robots that are within a Manhattan distance $d_{tr}$ of each other. In the case of multiple robots wanting to transfer loads with the same robot, the transfer that lowers cost the most is chosen. The planner used can be our Temporal A$^*$, M$^*$, or any other that finds paths to the robots' goals. 

The goal selection is revisited every time step to check for load transfers. The planned paths are saved until the goals of the robots change, to avoid unnecessary computation.

\section{Results}\label{sec:Results}

\subsection{Planner comparison}

We first measure the time required for planning by our Temporal A$^*$ implementation, regular A$^*$, and the M$^*$ variant. 
For five different maps with increasing number of robots, the three planners are tasked with finding paths to predetermined goal positions. 
Tests are conducted for movesets $S_5$ and $S_7$, which are illustrated in~\cref{fig:Neighbors}. 
Each test is given a maximum of one hour to solve. Results are shown in~\cref{fig:CompMaps}.

The maps are intentionally crowded to test the planners in challenging scenarios. From the results, it can be seen that larger values of $\varepsilon$ (a weighted  heuristic~\cite{pohl1970first}) do not substantially decrease the time required by regular A$^*$ or M$^*$ for most maps. Regular A$^*$ takes longer than M$^*$ for all maps, and fails consistently for the tests with $S_7$. On the contrary, Temporal A$^*$ consistently finds solutions much faster  \revision{thanks to the frequent reduction in neighboring nodes expanded}. 
Larger $\varepsilon$ do improve planning times in some cases.
Larger~$H$ values make Temporal A$^*$ take longer, which is expected as it means more steps are planned and thus considered when moving lower priority robots. However, lower values can prevent the planner from finding a solution, as can be seen in maps 1 and 5 with $S_5$, and in map~2 with $S_7$.

The performance of the Temporal~A$^*$ in terms of time steps required for reconfiguration of the maps, as well as the tiles traveled by all the robots, are shown in Table \ref{tab:PerfCompS7} for the $S_7$ case.
For the maps with 4 and 6 agents, Temporal A$^*$ can finish builds faster than both regular A$^*$ and M$^*$. In the other maps it takes slightly longer compared to other planners that found a solution.
In terms of total tiles traveled, Temporal A$^*$ consistently makes the agents travel longer paths as it sacrifices optimality to keep the planning decoupled.
This can be an important consideration for power requirements, for example.
Both of these issues can be partially addressed through the load transferring method.
As mentioned before, the size of $S$ is important for Temporal A$^*$ to find solutions.
It was not able to solve map~2 with $S_5$ for any set of parameters, because the limited moveset makes it impossible for the blue and red robots to let the other through, even if they switch priorities.
$S_7$~includes motions where the robots can turn their back feet first, offering more flexibility for the robots to get out of the way.

Unlike regular A$^*$ and M$^*$, where larger movesets increase planning time by increasing the number of neighbors each vertex has, for Temporal A$^*$ a larger moveset can actually decrease planning times. Its lower planning time performance makes it viable to use first and then switch to another planner like M$^*$ if a solution is not found.

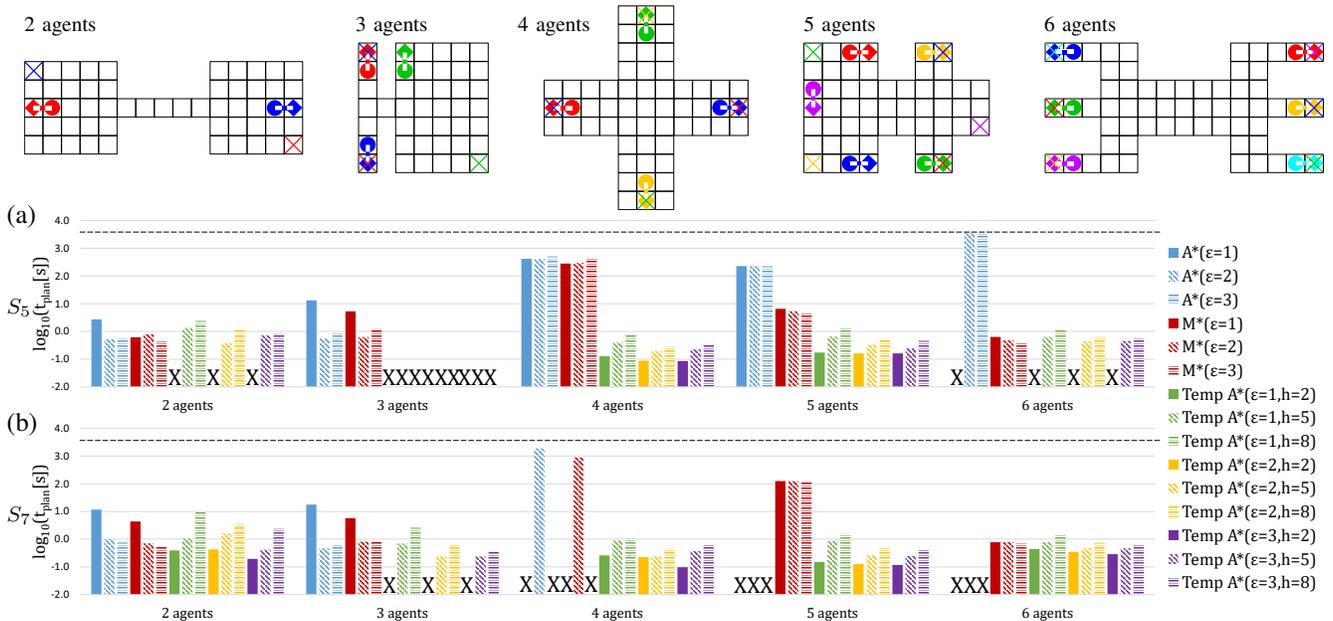
\begin{figure*}[tb]
	\setlength{\abovecaptionskip}{0pt}
	\centering
	\begin{overpic}[width=2\columnwidth]{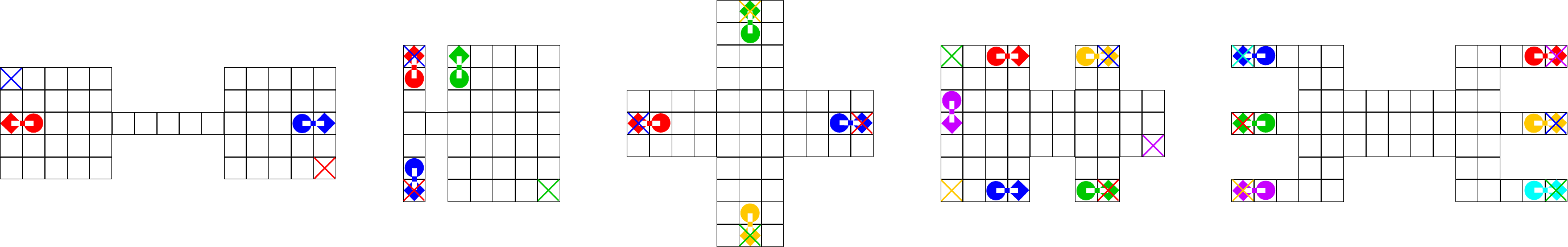}
		\put(0,13.75){\footnotesize 2 agents}
		\put(25.5,13.75){\footnotesize 3 agents}
		\put(38,13.75){\footnotesize 4 agents}
		\put(60,13.75){\footnotesize 5 agents}
		\put(78.5,13.75){\footnotesize 6 agents}
	\end{overpic}
	\begin{overpic}[width=2\columnwidth] {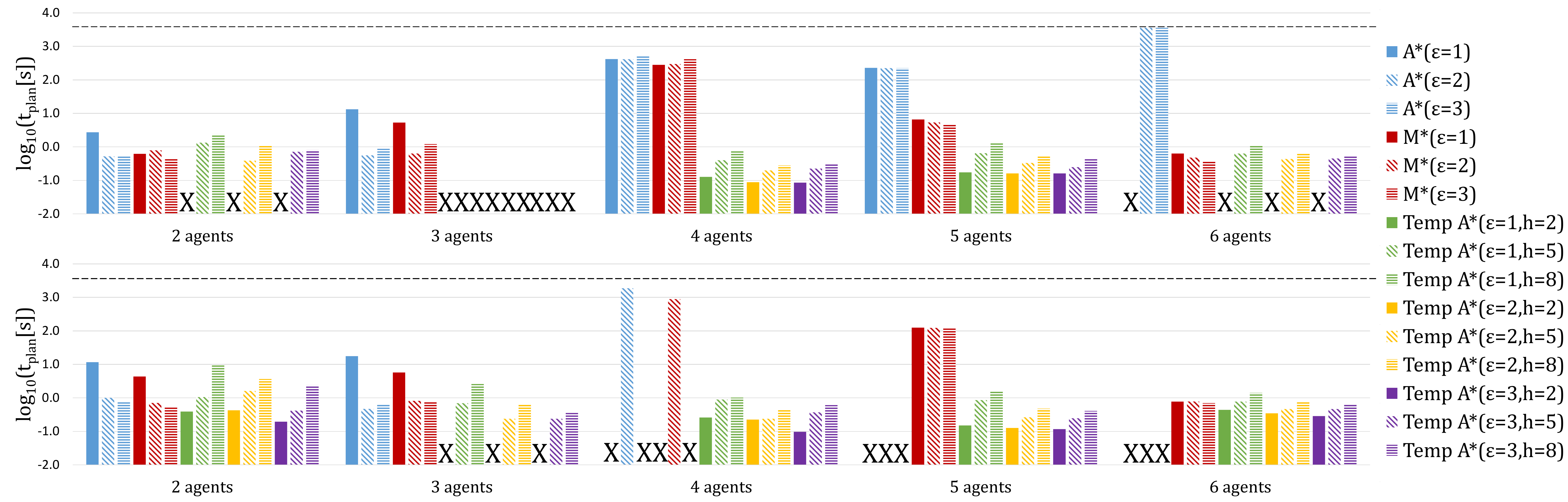}
		\put(-1.0,31){(a)}
		\put(-1.0,24){\footnotesize$S_5$}
		\put(-1.0,15){(b)}
		\put(-1.0,8){\footnotesize$S_7$}
	\end{overpic}
	\caption{\label{fig:CompMaps}
		Planning time $t_{\text{\emph{plan}}}$ for A$^*$, M$^*$, and our Temporal A$^*$ on crowded maps with increasing number of robots, shown at the top. The planners must move the robots' front feet (the circles) to the respective colored $\times$, the back foot can end up anywhere as long as it is a valid position. Results with $S_5$ in (a)  and  $S_7$ in (b). An X indicates the planner was unable to find a solution within the time limit (dashed line). Plots are shown in log scale.
	}
\end{figure*}

\begin{table*}[b]
\caption{Reconfiguration performance ($S_7$)}
\begin{center}
\begin{tabular}{|c|c|c|c|c|c|c|c|c|c|c|}
\hline
& \multicolumn{5}{c|}{Time steps} & \multicolumn{5}{c|}{Tiles traveled}\\
\hline
& \makecell{$m=2$} & \makecell{$m=3$} & \makecell{$m=4$} & \makecell{$m=5$} & \makecell{$m=6$} & \makecell{$m=2$} & \makecell{$m=3$} & \makecell{$m=4$} & \makecell{$m=5$} & \makecell{$m=6$}\\
\hline
\makecell{A$^*$($\epsilon=1$)} & \makecell{23} & \makecell{10} & \makecell{N/A} & \makecell{N/A} & \makecell{N/A} & \makecell{36} & \makecell{25} & \makecell{N/A} & \makecell{N/A} & \makecell{N/A}\\
\hline
\makecell{A$^*$($\epsilon=2$)} & \makecell{23} & \makecell{17} & \makecell{24} & \makecell{N/A} & \makecell{N/A} & \makecell{36} & \makecell{25} & \makecell{46} & \makecell{N/A} & \makecell{N/A}\\
\hline
\makecell{A$^*$($\epsilon=3$)} & \makecell{23} & \makecell{17} & \makecell{N/A} & \makecell{N/A} & \makecell{N/A} & \makecell{36} & \makecell{25} & \makecell{N/A} & \makecell{N/A} & \makecell{N/A}\\
\hline
\makecell{M$^*$($\epsilon=1$)} & \makecell{23} & \makecell{10} & \makecell{N/A} & \makecell{13} & \makecell{28} & \makecell{36} & \makecell{25} & \makecell{N/A} & \makecell{56} & \makecell{104}\\
\hline
\makecell{M$^*$($\epsilon=2$)} & \makecell{23} & \makecell{17} & \makecell{24} & \makecell{13} & \makecell{28} & \makecell{36} & \makecell{25} & \makecell{46} & \makecell{56} & \makecell{106}\\
\hline
\makecell{M$^*$($\epsilon=3$)} & \makecell{23} & \makecell{17} & \makecell{N/A} & \makecell{13} & \makecell{28} & \makecell{36} & \makecell{25} & \makecell{N/A} & \makecell{56} & \makecell{106}\\
\hline
\makecell{Temp A$^*$($\epsilon=1,h=2$)} & \makecell{30} & \makecell{N/A} & \makecell{14} & \makecell{17} & \makecell{28} & \makecell{59} & \makecell{N/A} & \makecell{50} & \makecell{80} & \makecell{165}\\
\hline
\makecell{Temp A$^*$($\epsilon=1,h=5$)} & \makecell{24} & \makecell{11} & \makecell{13} & \makecell{14} & \makecell{22} & \makecell{45} & \makecell{28} & \makecell{49} & \makecell{68} & \makecell{130}\\
\hline
\makecell{Temp A$^*$($\epsilon=1,h=8$)} & \makecell{24} & \makecell{11} & \makecell{13} & \makecell{14} & \makecell{22} & \makecell{45} & \makecell{28} & \makecell{49} & \makecell{68} & \makecell{130}\\
\hline
\makecell{Temp A$^*$($\epsilon=2,h=2$)} & \makecell{30} & \makecell{N/A} & \makecell{14} & \makecell{17} & \makecell{28} & \makecell{59} & \makecell{N/A} & \makecell{50} & \makecell{80} & \makecell{165}\\
\hline
\makecell{Temp A$^*$($\epsilon=2,h=5$)} & \makecell{24} & \makecell{11} & \makecell{13} & \makecell{14} & \makecell{22} & \makecell{45} & \makecell{28} & \makecell{49} & \makecell{68} & \makecell{130}\\
\hline
\makecell{Temp A$^*$($\epsilon=2,h=8$)} & \makecell{24} & \makecell{11} & \makecell{13} & \makecell{14} & \makecell{22} & \makecell{45} & \makecell{28} & \makecell{49} & \makecell{68} & \makecell{130}\\
\hline
\makecell{Temp A$^*$($\epsilon=3,h=2$)} & \makecell{30} & \makecell{N/A} & \makecell{14} & \makecell{17} & \makecell{28} & \makecell{59} & \makecell{N/A} & \makecell{50} & \makecell{80} & \makecell{165}\\
\hline
\makecell{Temp A$^*$($\epsilon=3,h=5$)} & \makecell{24} & \makecell{11} & \makecell{14} & \makecell{14} & \makecell{22} & \makecell{45} & \makecell{28} & \makecell{50} & \makecell{68} & \makecell{130}\\
\hline
\makecell{Temp A$^*$($\epsilon=3,h=8$)} & \makecell{24} & \makecell{11} & \makecell{14} & \makecell{14} & \makecell{22} & \makecell{45} & \makecell{28} & \makecell{52} & \makecell{68} & \makecell{103}\\
\hline
\end{tabular}
\label{tab:PerfCompS7}
\end{center}
\end{table*}

\subsection{Load transfer}

Load transfer often shortens total tiles traveled by the robots and time steps required for building, as illustrated in~\cref{fig:LoadTransferComps}. Different maps were created to perform a building sequence with and without load transfer enabled. These maps contain at least one long, single tile segment where the probability for deadlocks is high. Maps 1 to 4 have two robots, and maps 5 to 8 have three. Four tiles are moved in all maps with the exception of maps 3 and 6, where three and ten tiles are moved.

\begin{figure}[tb]
\setlength{\abovecaptionskip}{0pt}
\centering
    \begin{overpic}[width=\columnwidth] {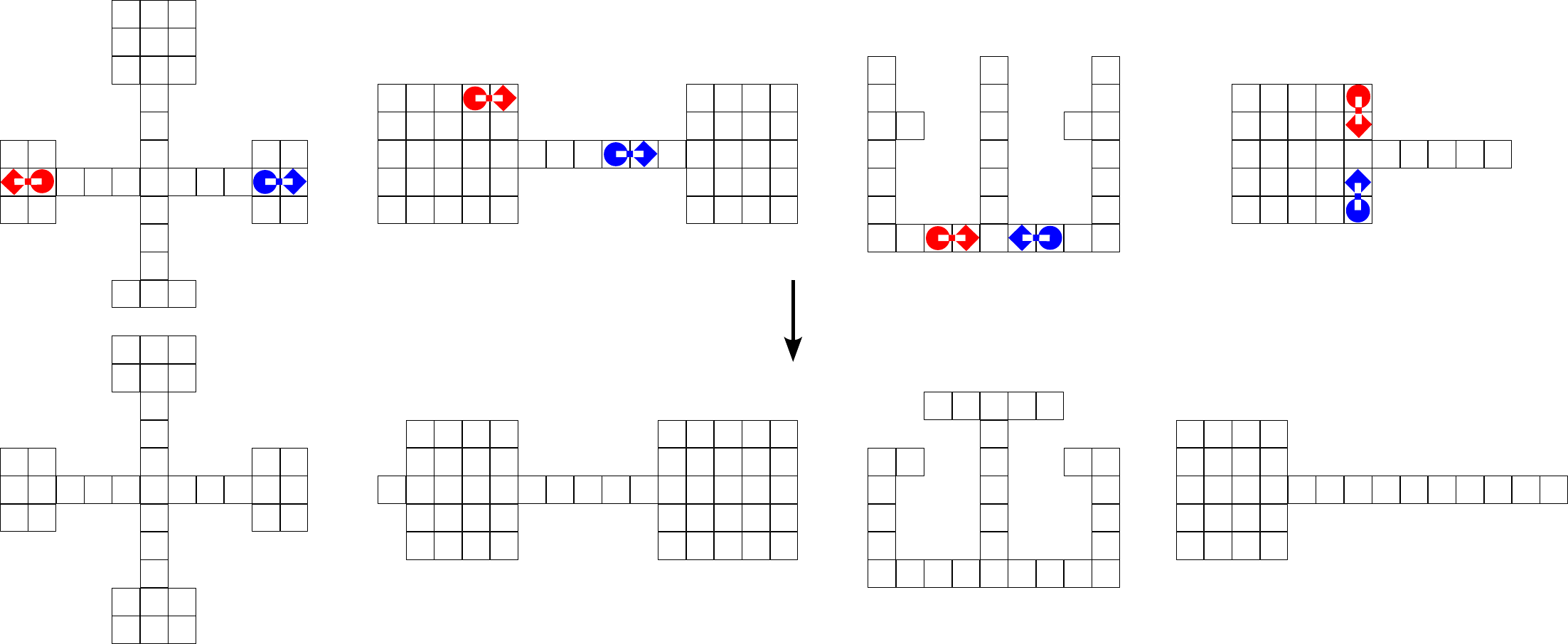}
         \put(13,38){\footnotesize Map 1}
         \put(25,38){\footnotesize Map 2}
         \put(55,38){\footnotesize Map 3}
         \put(75,38){\footnotesize Map 4}
    \end{overpic}
    \vspace{.1em}
    \adjincludegraphics[width=\columnwidth,trim={{0.1\width} {0.84\width} {0.18\width} {0.12\width}},clip] {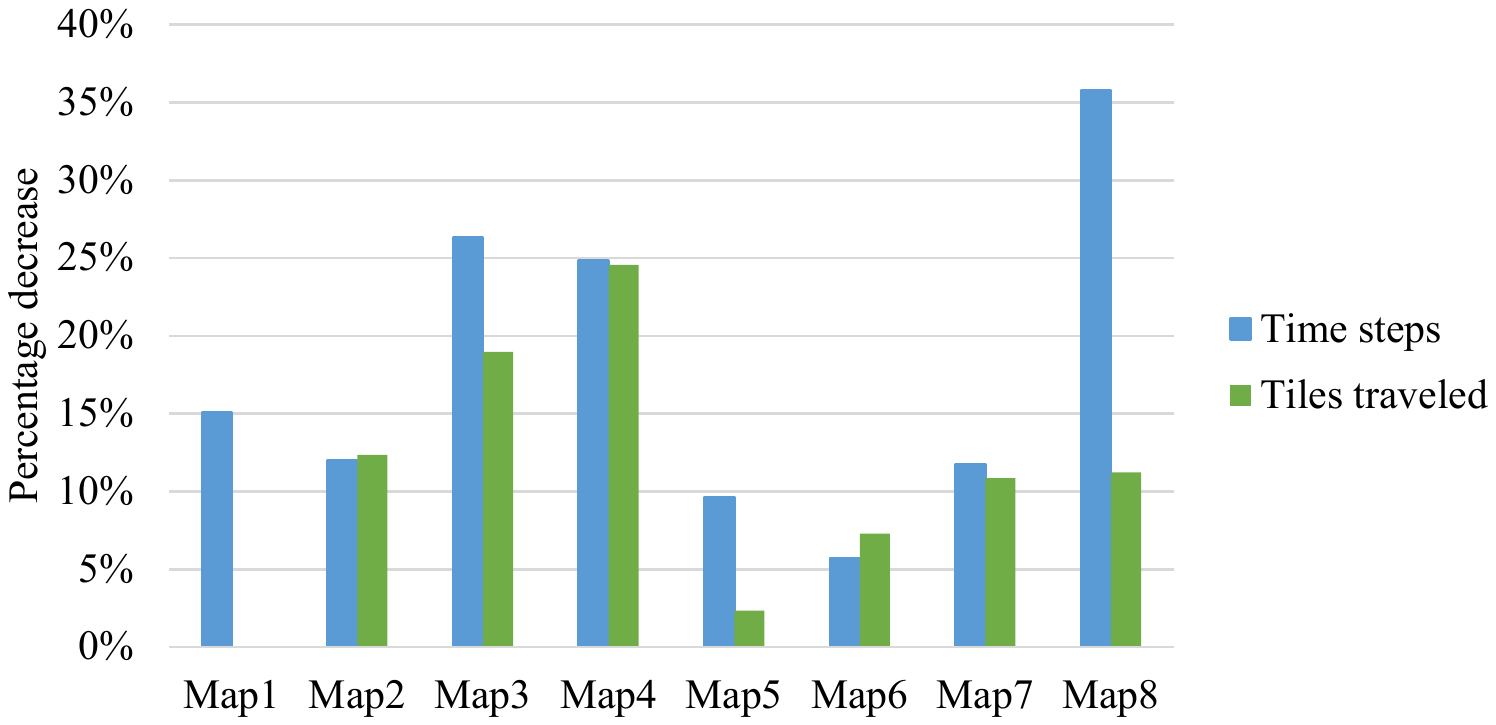}
    
    \begin{overpic}[width=\columnwidth] {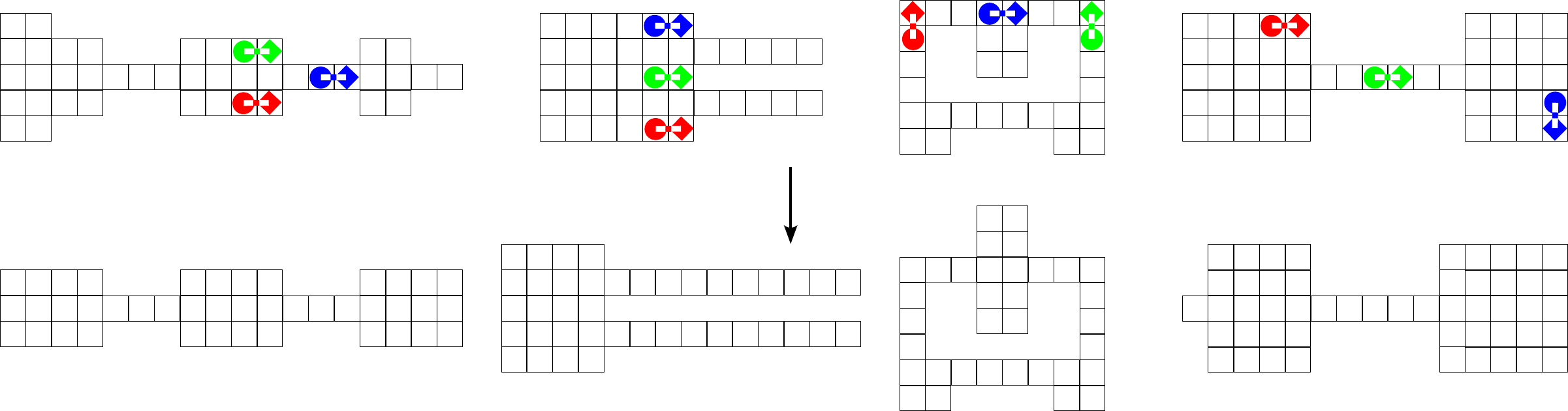}
         \put(0,-4){\footnotesize Map 5}
         \put(32,-4){\footnotesize Map 6}
         \put(57,-4){\footnotesize Map 7}
         \put(77,-4){\footnotesize Map 8}
    \end{overpic}
    \caption{\label{fig:LoadTransferComps}
    The load transfer strategy is tested on several maps. The initial configurations are shown on top and the final ones below them (the final positions of the robots do not matter as long as they are valid). The percentage decrease in time steps and tiles traveled, when employing load transferring, is shown in the plot. 
    }
\end{figure}

In these tests, our Temporal A$^*$ is used as the planner, with $\varepsilon=1$ and $h=10$, and $d_{tr}=5$. Load transferring's impact is how much it decreases the required time steps and total tiles traveled, compared to not employing it. 

The results for map 1 are interesting in that only the time steps were reduced. The same number of tiles were traveled but the robots did not wait as much for paths to be cleared. For the rest of the maps both metrics were reduced, significantly so for maps 3, 4 and 8. 

Although all maps used in these tests allow the building sequence to be completed without cooperation, there are many configurations that require cooperation to be completed, i.e. it is impossible to avoid deadlocks otherwise.

\section{Conclusions and Future Work}\label{sec:Conclusion}

We showed that several variants of the \textsc{\BILLE reconfiguration} problem are \NP-complete. 
In particular, the cooperative variant is proven to be hard, even for two robots. 
We compared three planners, and tested a load transfer strategy to reduce planning time and building time. 
The three planners are the regular A$^*$ algorithm, a multi-robot variant~M$^*$, and a priority-based Temporal A$^*$. 

The Temporal A$^*$ is faster at planning paths, and can benefit from larger configuration spaces unlike the other two. 
However, it is not complete and depends on the size of the moveset.
In some maps it finished the reconfiguration faster, although it usually results in a larger combined path cost (tiles traveled).
While the load transfer can significantly reduce travel costs and time required, it can be computationally expensive so it is limited to pairs of robots within a predefined distance of each other. 
Overall, we were able to plan paths and perform building sequences more efficiently utilizing both methods discussed.

Future work could focus on combining the temporal A* planning with the randomized sampling methods presented in our prior work, to further resolve deadlock situations in complex and crowded maps. \revision{We are also interested in extending these algorithms to 3D structures.}
Finally, it might be worth considering a transfer to distributed computation of local motion plans, as indicated in~\Cref{subsec:multi-robot-planning}.

\newpage
\bibliographystyle{IEEEtranDOI}
\bibliography{IEEEabrv,biblio.bib}
\end{document}